\documentclass{article}
\usepackage{spconf,amsmath,amsthm,amsfonts,bm,cite,graphicx,mathtools}
\usepackage{verbatim}
\usepackage{epstopdf}

\newtheorem{DEF}{Definition}

\newtheorem{PRO}{Proposition}
\newtheorem{EX}{Example}

\title{REGULARIZED RECOVERY BY MULTI-ORDER PARTIAL HYPERGRAPH TOTAL VARIATION}

\name{Ruyuan Qu \qquad Jiaqi He \qquad Hui Feng \qquad Chongbin Xu \qquad Bo Hu}

\address{School of Information Science and Technology\\
	Fudan University, Shanghai 200438, China\\
	Emails: \{19210720033, 18210720068, hfeng, chbinxu, bohu\}@fudan.edu.cn
}
\begin{document}
	\maketitle
	\begin{abstract}
		Capturing complex high-order interactions among data is an important task in many scenarios. A common way to model high-order interactions is to use hypergraphs whose topology can be mathematically represented by tensors. Existing methods use a fixed-order tensor to describe the topology of the whole hypergraph, which ignores the divergence of different-order interactions. In this work, we take this divergence into consideration, and propose a multi-order hypergraph Laplacian and the corresponding total variation. Taking this total variation as a regularization term, we can utilize the topology information contained by it to smooth the hypergraph signal. This can help distinguish different-order interactions and represent high-order interactions accurately.
	\end{abstract}

	\section{Introduction}
	\label{sec:intro}
	
	% 1 graph signal background
	Graphs are a useful tool to model pairwise interactions in structured datasets \cite{6494675,6808520,8347162}. For an $N$-instance dataset, the interaction between any two instances can be modeled as the edge weight in a graph, while data can be modeled as a graph signal.
	
	% 2 need for high-order interactions
	A high-order interaction is about influence or similarity among a group, which does not imply that there is a direct pairwise interaction between any two members of the group. In many scenarios, we need to deal with high-order interactions; for instance, interactions among strangers who buy the same product online. Due to the limit that graphs can only describe pairwise interactions, we need another tool to characterize complex high-order data structures appropriately.
	
	% 3 hypergraph: a good tool to describe high-order interactions
	As a generalization of graphs, hypergraphs can connect multiple vertices by an edge called hyperedge. Hypergraphs have been well applicable in signal processing to model high-order interactions \cite{7472914,8887197}.
	
	% smoothness and regularization
	Regularized signal recovery is a handy tool to reconstruct signals over both graphs and hypergraphs \cite{6494675,7605501,9001176,9146698}. Since we model data based on the similarity among them, we assume that signals evolve smoothly over the topology, which means signals are similar among neighboring vertices. Therefore, the smoothness measure of signals can be utilized as a regularization term in signal recovery.

	To represent the topology of hypergraphs, various tensor forms have been proposed in \cite{COOPER20123268,2017Spectra,2018On} which represent a hypergraph by a fixed-order tensor. Specifically, all hyperedges in an unweighted hypergraph can be represented by nonzero elements in a fixed-order tensor whose indices denote vertices in each hyperedge.% A Laplacian tensor is a differential operator over the hypergraph.
	
	However, it is tough to choose a suitable-order tensor to represent the topology of a hypergraph containing varying-cardinality hyperedges. Existing work chooses the maximum cardinality of hyperedges denoted by $M$ as the tensor order by adding vertices into hyperedges whose cardinalities are less than $M$ in \cite{2017Spectra}. This will introduce higher-order nonlinear relationships to our work, which may ignore the divergence of different-order interactions, and may bring difficulty to subsequent computations especially when cardinalities vary greatly. Moreover, choosing $M$ as the tensor order can not ensure the positive semidefiniteness of the Laplacian tensor, thus we can not propose a corresponding total variation form to measure the smoothness of the hypergraph signal.
	
	% 5 our work
	Instead, we obtain a new hypergraph Laplacian by decomposing a general hypergraph into a set of uniform partial hypergraphs and representing each of them by a Laplacian tensor. Based on the positive semidefiniteness of even order Laplacian tensors, we propose a total variation form to measure the smoothness of hypergraph signals, and recover hypergraph signals by means of the Laplacian regularization. In this way, we may distinguish the effect of different-order interactions and obtain a quite direct mathematical form of high-order interactions.
	% which is different from the matrix form Laplacian regularization in \cite{9146698}
	
	\section{Multi-order Partial Decomposition of Hypergraph}
	\label{sec:multiorder}
	
	\subsection{Tensors of Hypergraph}
	\label{ssec:dhypergraph}
	
	% s2 ss2 p2
	A hypergraph is a pair $\mathcal{H=(V,E)}$, where $\mathcal{V}=\left\lbrace v_1,\cdots,v_N\right\rbrace$ is a vertex set, and $\mathcal{E}=\left\lbrace\bm{e}_1,\cdots,\bm{e}_K\right\rbrace $ is a hyperedge set. Each element of $\mathcal{E}$ is a nonempty multi-element subset of $\mathcal{V}$ called hyperedge. The cardinality of each hyperedge $\bm{e}_k$ is denoted by $|\bm{e}_k|$  which is the number of vertices in it.
	
	To distinguish different-order interactions in a hypergraph, we first decompose hypergraph $\mathcal{H}$ into a partial hypergraph set $\mathcal{PH}\coloneqq\left\lbrace\mathcal{H}_{(c)},c\in\mathcal{C}\right\rbrace$, where $\mathcal{C}$ is the hyperedge cardinality set of a hypergraph, and each partial hypergraph $\mathcal{H}_{(c)}=(\mathcal{V},\mathcal{E}_{(c)})$ consists of all vertices and all $c$-cardinality hyperedges in hypergraph $\mathcal{H}$.
	
	\begin{EX}
		\label{example1}
		For hypergraph $\mathcal{H}$ in Fig. \ref{fig:hypergraphexample}, there are 4 hyperedges $\bm{e}_1$, $\bm{e}_2$, $\bm{e}_3$, $\bm{e}_4$ with cardinalities 3, 4, 2, 2 respectively.
		We can obtain 3 partial hypergraphs $\mathcal{H}_{(2)}$ containing $\bm{e}_3$, $\bm{e}_4$, $\mathcal{H}_{(3)}$ containing $\bm{e}_1$ and $\mathcal{H}_{(4)}$ containing $\bm{e}_2$, and all those partial hypergraphs consist of the whole vertex set $\mathcal{V}$.
	\end{EX}

	\begin{figure}[t]
		\centering
		\includegraphics[width=0.9\linewidth]{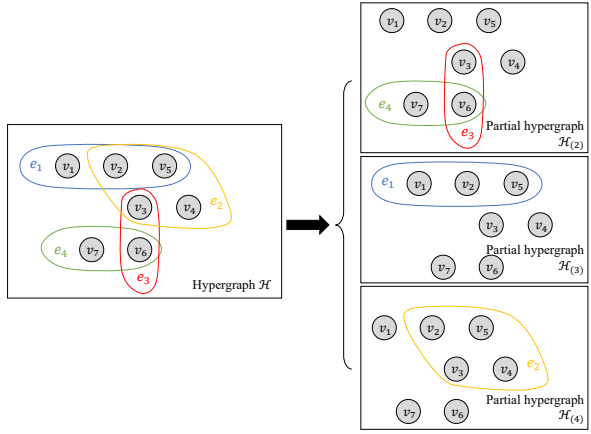}
		\caption{An example of multi-order partial hypergraph decomposition.}
		\label{fig:hypergraphexample}
	\end{figure}
	
	Tensors are a natural tool to describe high-order interactions mathematically due to the flexibility of the tensor order. To represent hypergraph $\mathcal{H}$, we define a new hypergraph Laplacian as a Laplacian tensor set $\mathcal{L}\coloneqq\left\lbrace\mathbf{L}_{(c)},c\in\mathcal{C}\right\rbrace$ whose element $\mathbf{L}_{(c)}$ contains the topology information of the $c$-uniform partial hypergraph $\mathcal{H}_{(c)}$. For $\forall c\in\mathcal{C}$,  the adjacency tensor $\mathbf{A}_{(c)}$ and Laplacian tensor $\mathbf{L}_{(c)}$ of the partial hypergraph $\mathcal{H}_{(c)}$ are formulated as follow according to \cite{COOPER20123268,2017Spectra}.
	
	% s2 ss2 p5
	\begin{DEF}[Adjacency tensor]
		\label{defat}
		%\cite{COOPER20123268}
		An $N$-vertex partial hypergraph $\mathcal{H}_{(c)}=(\mathcal{V},\mathcal{E}_{(c)})$ can be represented by a $c$th-order $N$-dimension adjacency tensor defined as
		\begin{equation}
		\label{eqA1}
		\mathbf{A}_{(c)}=(a_{i_1,\cdots,i_c}), 1\leq i_1,\cdots,i_c\leq N.
		\end{equation}
		For each hyperedge $\bm{e}=\left\lbrace v_{l_1},v_{l_2},\cdots,v_{l_c}\right\rbrace \in\mathcal{E}_{(c)}$, element
		\begin{equation}
		\label{eqA2}
		a_{i_1,\cdots,i_c}=\frac{1}{(c-1)!},
		\end{equation}
		where $i_1,\cdots,i_c$ are all permutations of $l_1,\cdots,l_c$.
		The other elements of $\mathbf{A}_{(c)}$ are zero.
	\end{DEF}

	% s2 ss2 p7
	\begin{DEF}[Laplacian tensor]
		\label{deflt}
		The Laplacian tensor of an $N$-vertex partial hypergraph $\mathcal{H}_{(c)}=(\mathcal{V},\mathcal{E}_{(c)})$ is defined as
		$\mathbf{L}_{(c)}=\mathbf{D}_{(c)}-\mathbf{A}_{(c)}$, where the degree tensor $\mathbf{D}_{(c)}$ is a $c$th-order $N$-dimension diagonal tensor with entries $d_{i,\cdots,i}=\sum_{i_2,\cdots,i_c=1}^{N}a_{i,i_2,\cdots,i_c}$ denoting the number of hyperedges in $\mathcal{E}_{(c)}$ containing vertex $v_i$ and the others equal to zero.
	\end{DEF}
	
	% s2 ss2 p8
	
	\subsection{Positive Semidefinite Tensor}
	\label{ssec:dtensor}
	
	\begin{comment}
		\begin{DEF}[The n-mode product \cite{Kolda2009Tensor}]
		The n-mode product of a tensor $\mathbf{X}\in\mathbb{R}^{I_1\times I_2\times\cdots\times I_N}$ with a matrix $\mathbf{U}\in\mathbb{R}^{J\times I_n}$ is denoted by $(\mathbf{X}\times_n\mathbf{U})\in\mathbb{R}^{I_1\times\cdots\times I_{n-1}\times J\times I_{n+1}\times\cdots\times I_N}$. Each element is obtained by
		\begin{equation}
		\label{eqNMP1}
		(\mathbf{X}\times_n\mathbf{U})_{i_1,\cdots,i_{n-1},j,i_{n+1},\cdots,i_N}=\sum_{i_n=1}^{I_n}x_{i_1,i_2,\cdots,i_N}u_{j,i_n}.
		\end{equation}
		\end{DEF}
	\end{comment}
	
	\begin{PRO}
		\label{pro1}
		An even order Laplacian tensor $\mathbf{L}$ with $M$ orders and $N$ dimensions is positive semidefinite, i.e. for $\forall\bm{f}\in\mathbb{R}^N$,
		\begin{equation}
		\begin{split}
		\label{eqPOLYN1}
		\mathbf{L}\bm{f}^M&\coloneqq\mathbf{L}\times_1\bm{f}^T\times_2\cdots\times_M\bm{f}^T\\&\;=\sum_{i_1,i_2,\cdots,i_M=1}^{N}l_{i_1,i_2,\cdots,i_M}f_{i_1}f_{i_2}\cdots f_{i_M}\geq 0,
		\end{split}
		\end{equation}
		where $\times_n$ is the $n$-mode product \cite{Kolda2009Tensor} of the tensor $\mathbf{L}$ with the vector $\bm{f}^T$.
	\end{PRO}

	\begin{proof}
		%\ref{defat} and
		According to Definition \ref{deflt}, as the degree of each vertex in partial hypergraph $\mathcal{H}_{(c)}$, each diagonal entry of $\mathbf{L}_{(c)}$ equals the sum of the absolute values of all off-diagonal entries in the corresponding slice of $\mathbf{L}_{c}$ given by
		\begin{equation}
		\begin{split}
		\label{eqLddt}
		&l_{i,\cdots,i}=d_{i,\cdots,i}=\sum_{i_2,\cdots,i_M=1}^{N}a_{i,i_2,\cdots,i_M}\\&=\sum_{i_2,\cdots,i_M} |l_{i,i_2,\cdots,i_M}|,\;\forall(i_2,\cdots,i_M)\neq(i,i,\cdots,i).
		\end{split}
		\end{equation}
		Therefore, a Laplacian tensor is diagonally dominated \cite{QI2014303}.
		By Theorem 3 in \cite{QI2014303}, as a diagonally dominated real symmetric tensor, an even order Laplacian tensor is positive semidefinite.
	\end{proof}

	\section{The Laplacian Regularization Estimation}
	\label{sec:LRM}
	
	\subsection{Problem Formulation}
	
	% s4 p1
	Consider an $N$-vertex hypergraph $\mathcal{H=\left( V,E\right)}$ with a real-valued hypergraph signal defined on $\mathcal{V}$ which can be represented as a vector $\bm{f}\in\mathbb{R}^N $. The $i$th element of the hypergraph signal $f_i$ represents the signal at the $i$th vertex in $\mathcal{V}$.
	
	% s3 p2
	Suppose that $S$ observations and the corresponding sampled vertex set $\mathcal{S}=\left\lbrace n_s,s=1,\cdots,S\right\rbrace$ are available. Then the observation model can be summarized as
	\begin{equation}
	\label{eqPF1}
	\bm{y}=\mathbf{\Psi}\bm{f}+\bm{\omega},
	\end{equation}
	where $\bm{y}=\left[ y_1,\cdots,y_S\right] ^T\in\mathbb{R}^S$ is the observation vector, $\mathbf{\Psi}\in\mathbb{R}^{S\times N}$ is the sampling operator with entries $\psi_{s,n_s}=1,\,s=1,\cdots,S$ and the others equal to zero, and $\bm{\omega}=\left[ \omega_1,\cdots,\omega_S\right] ^T$ is the noise vector. Given observations $\bm{y}$ and $\mathbf{\Psi}$, and the topology of hypergraph $\mathcal{H}$, our goal is to estimate signals at both observed and unobserved vertices.
	
	The hypergraph signal estimator can be obtained by the multi-order Laplacian regularization estimation (LRE) which solves a functional minimization problem formulated as
	\begin{equation}
	\label{eqOPTIM1}
	\hat{\bm{f}}=\mathop{\arg\min}_{\bm{f}\in\mathbb{R}^{N}}\ell(\mathbf{\Psi},\bm{f},\bm{y})+\lambda{\rm TV}(\bm{f}),
	\end{equation}
	where the loss function $\ell$ measures the error between estimators and observations, the regularization term $\rm TV$ is used to smooth the signal estimator and avoid overfitting based on the hypergraph topology, and the nonnegative parameter $\lambda$ determines preference between the loss and the smoothness. The detailed form of the regularization term $\rm TV$ will be given in subsection \ref{ssec:TV}.
	
	\subsection{Multi-order Total Variation over a Hypergraph}
	\label{ssec:TV}
	
	% s4 ss1 p1
	We measure signal differences among all vertices in each hyperedge by using high-order interactions directly rather than utilizing pairwise interactions extracted from hypergraphs \cite{Agarwal06higherorder,10.5555/2999792.2999883,10.1137/19M1291601,9001176}. We define the initial form of the total variation as
	\begin{equation}
	\label{eqLxM1}
	{\rm TV}(\bm{f})\coloneqq\sum_{c\in\mathcal{C}}{\rm TV}_{(c)}(\bm{f})=\sum_{c\in\mathcal{C}}\mathbf{L}_{(c)}{\bm{f}}^c,
	\end{equation}
	where ${\rm TV}_{(c)}$ denotes the total variation of signals among all $c$th-order interactions.
	
	% s4 ss1 p2
	However, to ensure the nonnegative property of the smoothness measure, we need to use only even order Laplacian tensors to represent the topology of the hypergraph.
	
	We first pretreat all odd-cardinality hyperedges by adding a vertex to each of them respectively to make their cardinalities even. Instead of choosing an existing vertex from the hyperedge as in \cite{2017Spectra}, we add an auxiliary vertex that $\mathcal{V}$ does not consist of to the hyperedge and define its signal as the arithmetic mean of signals at all vertices belonging to $\mathcal{V}$ in this hyperedge.
	
	\begin{EX}
		\label{example2}
		As the only odd-cardinality hyperedge in hypergraph $\mathcal{H}$ in Example \ref{example1}, $\bm{e}_1$ needs to be added a vertex $v_{aux}$ into to make its cardinality even as shown in Fig. \ref{fig:hypergraph-example-2}.
		The signal at $v_{aux}$ is denoted by $f(v_{aux})=\frac{f_1+f_2+f_5}{3}$.
	\end{EX}
	
	\begin{figure}[t]
		\centering
		\includegraphics[width=0.9\linewidth]{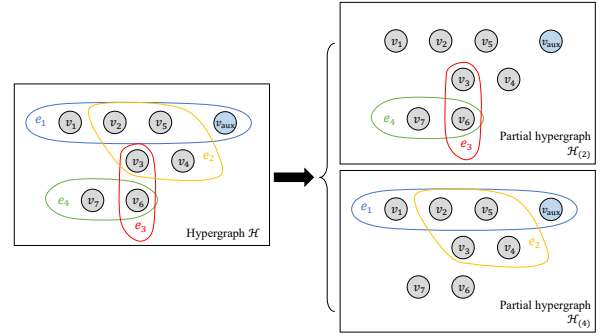}
		\caption{An example of multi-order partial hypergraph decomposition after pretreatment.}
		\label{fig:hypergraph-example-2}
	\end{figure}

	By pretreating the hypergraph, we can obtain an even cardinality set $\mathcal{C}$, a hypergraph Laplacian $\mathcal{L}=\left\lbrace\mathbf{L}_{(c)},c\in\mathcal{C}\right\rbrace$ containing only even order tensors, and a transformation matrix $\mathbf{T}$ to pretreat the hypergraph and its signal.
	
	Thus, the total variation \eqref{eqLxM1} can be rewritten as
	\begin{equation}
	\label{eqLxM}
	\begin{aligned}
	{\rm TV}(\bm{f})=\sum_{c\in\mathcal{C}}\mathbf{L}_{(c)}{(\mathbf{T}\bm{f})}^c=\sum_{c\in\mathcal{C}}\mathbf{L}_{(c)}{\tilde{\bm{f}}}^c.
	\end{aligned}
	\end{equation}
	\begin{comment}
		{\rm TV}(\bm{f})&=\sum_{c\in\mathcal{C}}{\rm TV}_{(c)}(\bm{f})=\sum_{c\in\mathcal{C}}\mathbf{L}_{(c)}{(\mathbf{T}\bm{f})}^c\\&=\sum_{c\in\mathcal{C}}\mathbf{L}_{(c)}{\tilde{\bm{f}}}^c.
	\end{comment}
	%The transformation matrix $\mathbf{T}$ is used to generate a hypergraph signal vector $\tilde{\bm{f}}={[\bm{f}^T,\bm{f}_{aux}^T]}^T$ which contains signal at both vertices in $\mathcal{V}$ and auxiliary vertices.
	Suppose the number of auxiliary vertices is $t$. The new hypergraph signal $\tilde{\bm{f}}=\mathbf{T}\bm{f}={[\bm{f}^T,\bm{f}_{aux}^T]}^T\in\mathbb{R}^{N+t}$ consists of signals at both vertices in $\mathcal{V}$ and auxiliary vertices. The transformation matrix $\mathbf{T}={[\mathbf{t}_1,\cdots,\mathbf{t}_{N+t}]}^T={[\mathbf{I}_N,\mathbf{T}_{aux}^{T}]}^T\in\mathbb{R}^{(N+t)\times N}$ where $\mathbf{t}_i^T\in\mathbb{R}^{1\times N}$ is the $i$th row of $\mathbf{T}$, and $\mathbf{I}_N$ is an $N\times N$ identity matrix. For $N+1\leq i\leq N+t$, suppose that auxiliary vertex $v_i$ belongs to a $c$-cardinality hyperedge $\bm{e}=\left\lbrace v_i,v_{n_1},\cdots,v_{n_{c-1}}\right\rbrace$ where vertices $v_{n_1},\cdots,v_{n_{c-1}}\in\mathcal{V}$. We have ${(\mathbf{t}_i)}_{n_j}=\frac{1}{c-1}$ for $j=1,\cdots,c-1$ and the remaining entries of $\mathbf{t}_i$ equal to 0 to make signal $\tilde{f}_i$ equal to the arithmetic mean of signals at the other $c-1$ vertices.
	
	\begin{EX}
		\label{example3}
		As shown in Fig. \ref{fig:hypergraph-example-2}, the total variation \eqref{eqLxM} over hypergraph $\mathcal{H}$ in Example \ref{example1} and Example \ref{example2} is
		\begin{equation*}
		\begin{aligned}
		{\rm TV}(\bm{f})&={\rm TV}_{(2)}(\bm{f})+{\rm TV}_{(4)}(\bm{f})\\&=\mathbf{L}_{(2)}{(\mathbf{T}\bm{f})}^2+\mathbf{L}_{(4)}{(\mathbf{T}\bm{f})}^4,
		\end{aligned}
		\end{equation*}
		where $\bm{f}\in\mathbb{R}^{7}$, $\mathbf{T}=
		\begin{bmatrix}
		&&&\mathbf{I}_7&&&\\
		\frac{1}{3}&\frac{1}{3}&0&0&\frac{1}{3}&0&0\\
		\end{bmatrix}\in\mathbb{R}^{8\times 7}$, and $\mathbf{L}_{(2)}$ and $\mathbf{L}_{(4)}$ are 8-dimension.
		This form only changes the odd cardinality in a small range by adding 1 to make it even. Specifically, the polynomial form of ${\rm TV}_{(4)}$ is
		\begin{equation*}
		\begin{aligned}
		{\rm TV}_{(4)}=&(f_1^4+f_2^4+f_5^4+\left( \frac{f_1+f_2+f_5}{3}\right) ^4-\\&4f_1f_2f_5\frac{f_1+f_2+f_5}{3}) +\\&(f_2^4+f_3^4+f_4^4+f_5^4-4f_2f_3f_4f_5).
		\end{aligned}
		\end{equation*}
		%This form only changes all odd cardinalities in a small range by adding 1 to make them even.
		
		Another total variation form of this hypergraph based on a fixed-order Laplacian in \cite{2017Spectra} is given by
		\begin{equation*}
		\begin{aligned}
		\mathbf{L}\bm{f}^4=&f_3^4+f_6^4-\frac{2}{7}(2f_3f_6^3+3f_3^2f_6^2+2f_3^3f_6)+\\&f_6^4+f_7^4-\frac{2}{7}(2f_6f_7^3+3f_6^2f_7^2+2f_6^3f_7)+\\&f_1^3+f_2^3+f_5^3-(f_1^2f_2f_5+f_1f_2^2f_5+f_1f_2f_5^2)+\\&f_2^4+f_3^4+f_4^4+f_5^4-4f_2f_3f_4f_5.
		\end{aligned}
		\end{equation*}
		This form represents a hyperedge $\bm{e}$ whose cardinality $c$ is less than the maximum cardinality of all hyperedges $M$ by an $M$th-order tensor. The cardinality of the hyperedge $\bm{e}$ needs to change to $M$ by choosing $M-c$ vertices in all possible ways from $\bm{e}$ and adding them into $\bm{e}$. If $M$ is much greater than $c$, the cardinality of the hyperedge changes too much, and the number of all choosing ways grows quite fast.
		
	\end{EX}
	
	% s4 ss1 p4
	When signals at each vertex and all its neighbors are equal, the total variation \eqref{eqLxM} is zero. If there exist multiple signal amplitudes differing greatly in a hyperedge, the total variation will be large, which indicates that the hypergraph signal is not smooth. Especially for a nonnegative hypergraph signal, taking this total variation as a regularization term can smooth the hypergraph signal by strongly penalizing signals which change greatly among multiple neighboring vertices.
	
	%When signals are the same among all hyperedges, the total variation of the hypergraph signal is zero. If there exist multiple signal amplitudes differing greatly in a hyperedge, the total variation will be quite large, which indicates that the hypergraph signal is not smooth. Especially for a nonnegative hypergraph signal, by applying AM-GM inequality \cite{2017Spectra}, the total variation \eqref{eqLxM} equals zero only when signals at each vertex and all its neighbors are equal. Taking this total variation as a regularization term can smooth the hypergraph signal by strongly penalizing signals which change greatly among multiple neighboring vertices.
	
	% s4 ss1 p5
	Using this total variation over a hypergraph can help capture high-order interactions, because we take the divergence of signal smoothness caused by different-cardinality hyperedges into consideration.
	%only change all odd cardinalities in a small range by adding 1 to make them even.
	Moreover, we retain the property that the auxiliary vertices only change the signal transmission within a hyperedge and have no influence on the way of signal transmission among different hyperedges.

	\subsection{A Gradient Descent Algorithm}
	
	We can get the estimator $\hat{\bm{f}}$ by using gradient descent method to solve \eqref{eqOPTIM1} formulated as
	\begin{equation}
	\begin{split}
	\label{eqgrad}
	&\bm{f}_{k+1}=\bm{f}_k-\eta\frac{\partial(\ell(\mathbf{\Psi},\bm{f},\bm{y})+\lambda{\rm TV}(\bm{f}))}{\partial\bm{f}}\bigg|_{\bm{f}=\bm{f}_k},
	\end{split}
	\end{equation}
	\begin{comment}
		\\=&\bm{f}_k-\eta\left(\frac{\partial\ell}{\partial\bm{f}}+\lambda\sum_{c\in\mathcal{C}}\frac{\partial\mathbf{L}_{(c)}{(\mathbf{T}\bm{f})}^c}{\partial\bm{f}}\right)\bigg|_{\bm{f}=\bm{f}_k}\\=&\bm{f}_k-\eta\left(\frac{\partial\ell}{\partial\bm{f}}+\lambda\sum_{c\in\mathcal{C}}c\mathbf{T}^T\left( \mathbf{L}_{(c)}{(\mathbf{T}\bm{f})}^{c-1}\right)\right)\bigg|_{\bm{f}=\bm{f}_k}
	\end{comment}
	where parameter $\eta>0$ is the stepsize, and the gradient of ${\rm TV}(\bm{f})$ taking the form \eqref{eqLxM} can be calculated as in \cite{QI20051302} applying the $n$-mode products \cite{Kolda2009Tensor} to the tensor $\mathbf{L}_{(c)}$ with the vector $(\mathbf{T}\bm{f})^T$ among the last $c-1$ modes given by 
	\begin{equation}
		\begin{split}
		\frac{\partial{\rm TV}(\bm{f})}{\partial\bm{f}}=\sum_{c\in\mathcal{C}}c\mathbf{T}^T\left( \mathbf{L}_{(c)}{(\mathbf{T}\bm{f})}^{c-1}\right).
		\end{split}
	\end{equation}
	
	\begin{comment}
		\frac{\partial{\rm TV}}{\partial\bm{f}}=\sum_{c\in\mathcal{C}}\frac{\partial\mathbf{L}_{(c)}{(\mathbf{T}\bm{f})}^c}{\partial\bm{f}}=\sum_{c\in\mathcal{C}}c\mathbf{T}^T\left( \mathbf{L}_{(c)}{(\mathbf{T}\bm{f})}^{c-1}\right)
	\end{comment}
	 \begin{comment}
	 	and the $i$th entry of $\mathbf{L}_{(c)}{(\mathbf{T}\bm{f})}^{c-1}$ is given by
	 	\begin{equation}
	 	\begin{split}
	 	\left(\mathbf{L}_{(c)}{(\mathbf{T}\bm{f})}^{c-1}\right)_i&=\sum_{i_2,\cdots,i_c=1}^{N+t}(\mathbf{L}_{(c)})_{i,i_2,\cdots,i_c}(\mathbf{T}\bm{f})_{i_2}\cdots(\mathbf{T}\bm{f})_{i_c}.
	 	\end{split}
	 	\end{equation}
	 \end{comment}

	\section{Numerical Results}
	\label{sec:empirical}
	
	We provide a simulation example for $\lambda=0.001$ over the zoo dataset \cite{Dua:2019} which consists of 101 instances and 17 features for each instance.
	
	We can generate an $N$-vertex hypergraph by choosing $N$ instances randomly from the dataset. We choose feature whether animals have feathers as the hypergraph signal by setting signal value 0.95 for true and 0.05 for false, and suppose that all observations are noiseless. The other 16 features are used to generate the hypergraph topology. If instances contain the attribute for a Boolean feature or behave the same in a multiple-value feature, we will put the corresponding vertices in a hyperedge. We then sample vertices randomly from the $N$ vertices according to the fraction of observations.
	
	Given the hypergraph topology and observations, we can use LRE by choosing the cross entropy function as the loss $\ell$ to estimate the signal. After obtaining the estimator $\hat{\bm{f}}$, we recover the signal by setting a threshold 0.5 and classifying vertices into two parts. We compare this recovery accuracy of unobserved signals with the results of LP-GSP \cite{8347162,2002Learning} and LP-HGSP \cite{8887197}. LP-GSP represents data structure by a graph and propagates labels over the graph based on the proximity of vertices. LP-HGSP uses fixed-order adjacency tensors to represent hypergraphs and processes signals by utilizing CP-ORTHO decomposition \cite{afshar2017cp} which aims to decompose a tensor into rank-one tensors as orthogonal as possible.
	
	We set $N=30$ and fractions of observations from 0.4 to 0.7. By taking the average of 10000-trial results, we obtain the accuracy of unobserved signal recovery. As shown in Fig. \ref{fig:sig9505fraction7n30compare2methods}, LRE outperforms LP-GSP and LP-HGSP, which indicates the total variation \eqref{eqLxM} contains more appropriate high-order interaction information than both matrix representations of graphs and fixed-order tensor representations of hypergraphs in this situation. LP-HGSP performs worse than LP-GSP probably because of the uncertainty of CP-ORTHO decomposition. So by multi-order decomposition of hypergraph, LRE can utilize topology information better in some ways.
	
	\begin{figure}[t]
		\centering
		\includegraphics[width=0.7\linewidth]{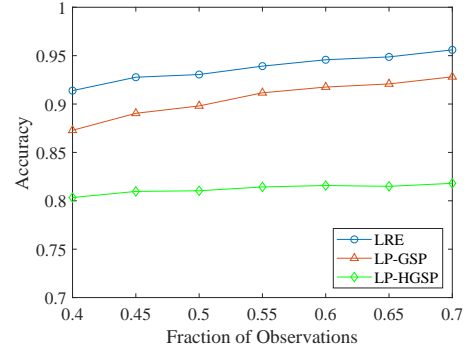}
		\caption{Performance of unobserved signal recovery with different fractions of observations.}
		\label{fig:sig9505fraction7n30compare2methods}
	\end{figure}

	%\vfill
	%\pagebreak
	
	\section{Conclusion}
	\label{sec:conclusion}
	
	In this work, we propose a hypergraph Laplacian in the form of a multi-order Laplacian tensor set and define the corresponding total variation over a hypergraph, which can help distinguish different-order interactions and capture high-order interactions in some ways.
	We use LRE to recover the hypergraph signal by taking the total variation as a regularization term.
	We also provide a numerical example to illustrate the advantages of LRE.
	All these facts indicate that LRE is a useful and feasible method in hypergraph signal recovery.
	
	\section{Acknowledgment}
	This work was supported by the Shanghai Municipal Natural Science Foundation (No. 19ZR1404700), the 2020 Okawa Foundation Research Grant, and the Fudan-Zhuhai Innovation Institute.
	
	\vfill\pagebreak
	
	\bibliographystyle{IEEEbib}
	\bibliography{refs}

\begin{thebibliography}{10}

\bibitem{6494675}
D.~I. {Shuman}, S.~K. {Narang}, P.~{Frossard}, A.~{Ortega}, and
  P.~{Vandergheynst},
\newblock ``The emerging field of signal processing on graphs: Extending
  high-dimensional data analysis to networks and other irregular domains,''
\newblock {\em IEEE Signal Processing Magazine}, vol. 30, no. 3, pp. 83--98,
  2013.

\bibitem{6808520}
A.~{Sandryhaila} and J.~M.~F. {Moura},
\newblock ``Discrete signal processing on graphs: Frequency analysis,''
\newblock {\em IEEE Transactions on Signal Processing}, vol. 62, no. 12, pp.
  3042--3054, 2014.

\bibitem{8347162}
A.~{Ortega}, P.~{Frossard}, J.~{Kovačević}, J.~M.~F. {Moura}, and
  P.~{Vandergheynst},
\newblock ``Graph signal processing: Overview, challenges, and applications,''
\newblock {\em Proceedings of the IEEE}, vol. 106, no. 5, pp. 808--828, 2018.

\bibitem{7472914}
S.~{Barbarossa} and M.~{Tsitsvero},
\newblock ``An introduction to hypergraph signal processing,''
\newblock in {\em 2016 IEEE International Conference on Acoustics, Speech and
  Signal Processing (ICASSP)}, 2016, pp. 6425--6429.

\bibitem{8887197}
S.~{Zhang}, Z.~{Ding}, and S.~{Cui},
\newblock ``Introducing hypergraph signal processing: Theoretical foundation
  and practical applications,''
\newblock {\em IEEE Internet of Things Journal}, vol. 7, no. 1, pp. 639--660,
  2020.

\bibitem{7605501}
D.~{Romero}, M.~{Ma}, and G.~B. {Giannakis},
\newblock ``Kernel-based reconstruction of graph signals,''
\newblock {\em IEEE Transactions on Signal Processing}, vol. 65, no. 3, pp.
  764--778, 2017.

\bibitem{9001176}
H.~C. {Nguyen} and H.~{Mamitsuka},
\newblock ``Learning on hypergraphs with sparsity,''
\newblock {\em IEEE Transactions on Pattern Analysis and Machine Intelligence},
  pp. 1--1, 2020.

\bibitem{9146698}
W.~{Hao}, S.~{Pang}, J.~{Zhu}, and Y.~{Li},
\newblock ``Self-weighting and hypergraph regularization for multi-view
  spectral clustering,''
\newblock {\em IEEE Signal Processing Letters}, vol. 27, pp. 1325--1329, 2020.

\bibitem{COOPER20123268}
J.~{Cooper} and A.~{Dutle},
\newblock ``Spectra of uniform hypergraphs,''
\newblock {\em Linear Algebra and its Applications}, vol. 436, no. 9, pp. 3268
  -- 3292, 2012.

\bibitem{2017Spectra}
A.~{Banerjee}, A.~{Char}, and B.~{Mondal},
\newblock ``Spectra of general hypergraphs,''
\newblock {\em Linear Algebra and its Applications}, vol. 518, pp. 14--30,
  2017.

\bibitem{2018On}
X.~Ouvrard, J.~M.~Le Goff, and S.~Marchand-Maillet,
\newblock ``On adjacency and e-adjacency in general hypergraphs: Towards a new
  e-adjacency tensor,''
\newblock {\em Electronic Notes in Discrete Mathematics}, vol. 70, pp. 71--76,
  2018.

\bibitem{Kolda2009Tensor}
T.~G. Kolda and B.~W. Bader,
\newblock ``Tensor decompositions and applications,''
\newblock {\em Siam Review}, vol. 51, no. 3, pp. 455--500, 2009.

\bibitem{QI2014303}
L.~{Qi} and Y.~{Song},
\newblock ``An even order symmetric b tensor is positive definite,''
\newblock {\em Linear Algebra and its Applications}, vol. 457, pp. 303 -- 312,
  2014.

\bibitem{Agarwal06higherorder}
S.~Agarwal, K.~Branson, and S.~Belongie,
\newblock ``Higher order learning with graphs,''
\newblock in {\em In ICML ’06: Proceedings of the 23rd international
  conference on Machine learning}, 2006, pp. 17--24.

\bibitem{10.5555/2999792.2999883}
M.~Hein, S.~Setzer, L.~Jost, and S.~S. Rangapuram,
\newblock ``The total variation on hypergraphs - learning on hypergraphs
  revisited,''
\newblock in {\em Proceedings of the 26th International Conference on Neural
  Information Processing Systems - Volume 2}, Red Hook, NY, USA, 2013, NIPS'13,
  p. 2427–2435, Curran Associates Inc.

\bibitem{10.1137/19M1291601}
J.~{Chang}, Y.~{Chen}, L.~{Qi}, and H.~{Yan},
\newblock ``Hypergraph clustering using a new laplacian tensor with
  applications in image processing,''
\newblock {\em SIAM Journal on Imaging Sciences}, vol. 13, no. 3, pp.
  1157--1178, 2020.

\bibitem{QI20051302}
L.~Qi,
\newblock ``Eigenvalues of a real supersymmetric tensor,''
\newblock {\em Journal of Symbolic Computation}, vol. 40, no. 6, pp. 1302 --
  1324, 2005.

\bibitem{Dua:2019}
D.~Dua and C.~Graff,
\newblock ``{UCI} machine learning repository,'' 2017.

\bibitem{2002Learning}
X.~Zhu and Z.~Ghahramani,
\newblock ``Learning from labels and unlabeled data with label propagation,''
\newblock {\em CMU CALD Tech Report}, vol. 3175, no. 2004, pp. 237--244, 2002.

\bibitem{afshar2017cp}
A.~Afshar, J.~C. Ho, B.~Dilkina, I.~Perros, E.~B. Khalil, L.~Xiong, and
  V.~Sunderam,
\newblock ``Cp-ortho: An orthogonal tensor factorization framework for
  spatio-temporal data,''
\newblock in {\em Proceedings of the 25th ACM SIGSPATIAL International
  Conference on Advances in Geographic Information Systems}, 2017, pp. 1--4.

\end{thebibliography}
	
\end{document}